\documentclass[journal,twoside,web]{ieeecolor}
\usepackage{tmi}
\usepackage{cite}
\usepackage{amsmath,amssymb,amsfonts}
\usepackage{algorithmic}
\usepackage{graphicx}
\usepackage{textcomp}
\usepackage{graphicx} % 调整表格缩放
\usepackage{booktabs} % 优化表格线
\usepackage{multirow}
\usepackage{array}
\usepackage{rotating}
\usepackage{colortbl}
\usepackage{makecell}
\usepackage{hhline}
\usepackage{ragged2e}
\usepackage{algorithm}
\usepackage{hyperref}
\hypersetup{hidelinks,
	colorlinks=true,
	allcolors=black,
	pdfstartview=Fit,
	breaklinks=true}
\newtheorem{theorem}{Theorem}  
\newenvironment{proof}[1][Proof]{%
	\par\indent\textit{#1:} \ignorespaces
}{%
	\hfill$\blacksquare$
}
\makeatletter
\def\BibTeX{{\rm B\kern-.05em{\sc i\kern-.025em b}\kern-.08em
    T\kern-.1667em\lower.7ex\hbox{E}\kern-.125emX}}
\begin{document}
\title{FrequencyCT: Frequency Domain Self-supervised Low-dose CT Denoising}
\author{Guoquan Wei, Liu Shi, Chong Chen, Qiegen Liu, \IEEEmembership{Senior Member, IEEE}
	\thanks{This work was supported by the National Natural Science Foundation of China (Grant: 621220033, 12322117 and 12288201), National Key Research and Development Program of China (Grant: 2023YFA1009300), Nanchang University Youth Talent Training Innovation Fund Project (Grant: XX202506030012), Early-Stage Young Scientific and Technological Talent Training Foundation of Jiangxi Province (Grant: 20252BEJ730005). (G. Wei is the first author.) (Co-corresponding authors: Q. Liu and L. Shi.)}
	\thanks{G. Wei, L. Shi, and Q. Liu are with the School of Information
		Engineering, Nanchang University, Nanchang 330031, China (email: \{liuqiegen, shiliu\}@ncu.edu.cn, guoquanwei@email.ncu.edu.cn).}
	\thanks{C. Chen is with the SKLMS, ICMSEC, Academy of Mathematics and Systems Science, Chinese Academy of Sciences, Beijing 100190, China (email: chench@lsec.cc.ac.cn).}
}
\maketitle
\begin{abstract}
	Despite extensive research on computed tomography (CT) denoising, few studies exploit projection-domain data characteristics to mitigate noise correlation. To bridge this gap, this work proposes FrequencyCT, the first zero-shot self-supervised method for pseudo-sample generation in the frequency domain for low-dose CT denoising. Specifically, by exploiting the distinct frequency-domain distributions of noise and true signal, a regional low-frequency anchoring technique is proposed. Applying phase-preserving noise and mask perturbations to the high-frequency region generates pseudo-samples for self-supervision. Driven by the exponential correlation between noise variance of noisy projections and the underlying true signal, consistent data truncation is applied to the generated samples to stabilize optimization gradients. Evaluation results on multiple public and real datasets confirm the clinical application potential of this research, which provides an innovative perspective for the field of denoising. The code is available at: \url{https://github.com/yqx7150/FrequencyCT}.
\end{abstract}


\section{Introduction}
\IEEEPARstart{C}{omputed} tomography (CT) has become a powerful tool in modern medical diagnosis \cite{brenner2007computed}. To reduce the risk of radiation-induced cancer, low-dose CT (LDCT) is the preferred method used by medical institutions \cite{smith2009radiation}. However, reducing the radiation dose inherently limits photon detection and amplifies electronic noise, a persistent challenge despite the mitigating effects of emerging photon-counting CT (PCCT) \cite{meloni2023photon}. Unlike stochastic noise in natural images, the noise in LDCT, following logarithmic transformation and filtered back projection (FBP), typically manifests as highly correlated, globally extended streak artifacts. This structured noise severely degrades soft tissue contrast, obscures minute pathologies, and compromises downstream analysis \cite{bosch2023risk}. Consequently, developing robust and efficient denoising algorithms for low-dose scenarios has become an urgent clinical necessity to guarantee diagnostic fidelity under strictly constrained radiation environments.

Denoising research on CT has achieved substantial progress through various methodological paradigms. Early explorations primarily relied on traditional filtering and iterative algorithms \cite{wen2008iterative}, such as non-local means \cite{buades2011non}, dictionary learning \cite{dong2011sparsity}, total variation \cite{vogel1996iterative}, and BM3D \cite{dabov2007image}. While these approaches exhibit adequate denoising performance, they are inherently constrained by operational assumptions. Furthermore, these methods typically require extensive empirical parameter tuning and prolonged iterative optimization times, which collectively hinder their deployment in emergency clinical scenarios. Supervised deep learning has accelerated field development by mapping noisy inputs to clean targets \cite{shan2019competitive}, but the heavy reliance on precisely paired data continues to bottleneck these approaches. More recently, diffusion models \cite{yang2023diffusion,10776993,10506793} have elevated unsupervised denoising to unprecedented levels. These generative frameworks still require further optimization to address extended training durations. To reduce reliance on external reference data, self-supervised learning \cite{lehtinen2018noise2noise,niu2022noise} achieves efficient denoising by deeply mining the statistical properties of large-scale unaligned noisy observations. However, this pre-training-dependent self-supervised paradigm still requires extensive additional noisy data for training.

In the evolution of self-supervised learning, zero-shot denoising has garnered significant attention due to its minimal reliance on training samples \cite{jaiswal2020survey}. Early explorations in this field primarily leveraged the deep image prior (DIP) \cite{ulyanov2018deep} mechanism, whose core idea is to exploit the inductive bias inherent in network architectures to constrain the denoising process. Subsequently, ZS-N2N \cite{mansour2023zero} constructed supervised samples via local diagonal pixel downsampling, introducing an unavoidable distribution shift between the training and testing phases. More recently, Pixel2Pixel \cite{ma2025pixel2pixel} generates a large number of pseudo-instances by deeply mining non-local self-similarity and introducing a random sampling strategy, thereby partially breaking the spatial noise correlation bottleneck. However, these statistically driven approaches are primarily restricted to the image domain, failing to leverage the projection domain perspective that is more conducive to decoupling noise. As illustrated in Fig. \ref{fig_projection} (a), their local averaging and similarity fail to effectively break noise correlations when confronting projection signals characterized by strong signal dependencies and highly non-stationary fluctuations. Fortunately, a growing body of recent research has emerged to leverage CT imaging mechanism for projection-domain denoising. For example, Choi \emph{et al.} \cite{choi2023self} exploited the geometric correlation between adjacent cone-beam CT projections for self-supervised constraints.  Unal \emph{et al.} \cite{Unal2024Proj2ProjSL} utilized physical blind zones in sinograms, but the frequent invocation of forward- and back-projection operators compromises computational efficiency. Meanwhile, An \emph{et al.} \cite{An_2024} partitioned detector arrays for downsampling, whereas Shi \emph{et al.} \cite{shi2025zeroshotlowdosectdenoising} leveraged projection redundancy from the conjugate theorem to generate pseudo-samples. As shown in Fig. \ref{fig_projection}~(b), these works primarily focus on suppressing noise by exploiting the inherent physical redundancy in imaging geometry. However, such methods still have obvious limitations. Their idealized assumptions about the imaging scanning trajectory and high sensitivity to data interpolation fluctuations hinder their widespread applicability. Nevertheless, their deep exploration of the underlying CT imaging mechanism points the way forward for the denoising field. This motivates researchers to explore a more stable domain to thoroughly decouple noise correlations.
\begin{figure}[!htbp]
	\centering
	\includegraphics[width=\linewidth]{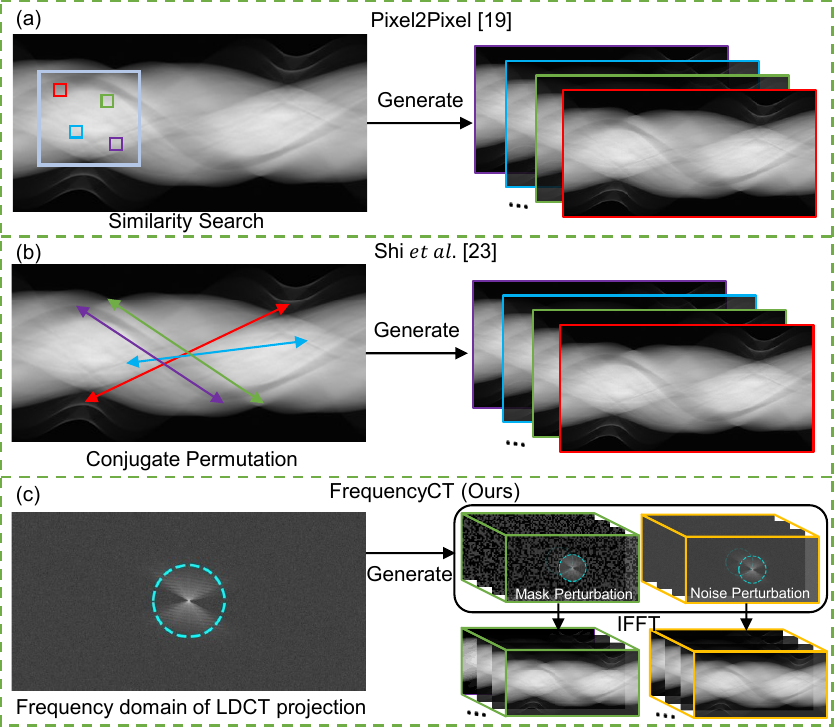}
	\caption{Schematic illustration of different zero-shot self-supervised denoising methods. (a) Schematic illustration of applying image-domain statistical pseudo-sample generation methods such as Pixel2Pixel \cite{ma2025pixel2pixel} directly to the projection domain. (b) Schematic illustration of generating pseudo-samples for denoising utilizing projection imaging mechanisms such as the approach by Shi \emph{et al.} \cite{shi2025zeroshotlowdosectdenoising}. (c) Our proposed FrequencyCT for pseudo-sample generation utilizing frequency-domain noise characteristics in the projection domain.}
	\label{fig_projection}
\end{figure}

Based on the above analysis, this work proposes FrequencyCT, a zero-shot self-supervised framework operating in Fourier domain. As illustrated in Fig. \ref{fig_projection} (c), unlike the aforementioned methods, it leverages the distinct frequency-domain distributions of noise and true signal to generate pseudo-samples, effectively decoupling spatial noise correlations while preserving fine details. This work analyzes high- and low-frequency components in the projection domain, recognizing that the high-frequency region is highly susceptible to noise and exploiting it to construct pseudo-samples. Specifically, a large number of pseudo-samples are generated by perturbing the high-frequency region through phase-preserving noise perturbation (PPNP) and phase-preserving mask perturbation (PPMP). These samples, generated by both mechanisms, have the same anatomical information macroscopically, and the orthogonal generation strategy breaks the spatial noise correlation. Simultaneously, theoretical derivations demonstrate that consistent truncation of training inputs and targets effectively isolates extreme outliers. This prevents high-variance fluctuations from disrupting gradient optimization, thereby facilitating robust network convergence. Experimental results confirm that the proposed method demonstrates effectiveness and robustness in a wide range of experiments, including simulated noise, real clinical data, and collected mouse data.

The main contributions of this work are as follows:
\begin{itemize}{}{}
	\item{We propose FrequencyCT, a novel zero-shot denoising method featuring an innovative frequency-domain pseudo-sample generation strategy. By leveraging the distinct frequency-domain distributions of noise and true signal, this method generates pseudo-samples through noise and mask perturbations on the original noisy data, effectively decoupling spatial noise correlations.}
	\item{We theoretically verify that the noise variance in LDCT projections escalates exponentially with the value of the underlying true signal. This property motivates a consistent data truncation on the generated samples before training, which effectively stabilizes optimization gradients and accelerates network convergence.}
\end{itemize}

The rest of this paper is organized as follows: Section \ref{Methodology} describes a new zero-shot denoising method, FrequencyCT. Section \ref{Experiments} presents the experimental results. Finally, Sections \ref{Discussion and Conclusion} and \ref{Conclusion} discuss and conclude the entire study.

%\begin{figure*}[!htbp]
%	\centering
%	\includegraphics[width=\linewidth]{figures/fig1.pdf}
%	\caption{The overall flowchart of FrequencyCT. (a) Description of generating pseudo-samples using noise perturbation and mask perturbation. (b) Detailed description of data truncation training and inference.}
%	\label{fig1}
%\end{figure*}

\section{Methodology}
\label{Methodology}
Our methodological analysis consists of three key components: An analysis of the distribution characteristics of noise and true signal in the frequency domain (Sec. \ref{Motivation}), a pseudo-sample generation strategy utilizing phase-preserving noise and mask perturbation motivated by this analysis (Sec. \ref{Proposed FrequencyCT}), and a theoretical proof demonstrating that projection-domain truncated training isolates extreme statistical outliers to enhance network optimization (Sec. \ref{Theoretical Analysis}).
\subsection{Motivation}
\label{Motivation}
To overcome noise correlation in the projection domain, this work shifts the perspective to the frequency domain, where true signals and noise exhibit distinct distributions. For a given two-dimensional LDCT projection matrix $p_{ld}\in\mathbb{R}^{H\times W}$, its corresponding two-dimensional discrete Fourier transform can be expressed as:
\begin{equation}
	\begin{aligned}
		P_{ld}(u,v)&=\mathcal{F}(p_{ld})
		\\&=\frac{1}{\sqrt{HW}}\sum_{h=0}^{H-1}\sum_{w=0}^{W-1}p_{ld}(h,w)e^{-j2\pi\left(\frac{uh}{H}+\frac{vw}{W}\right)},
	\end{aligned}
	\label{eq1}
\end{equation}
where $P_{ld}\left(u,v\right)$ is the complex spectrum. In polar coordinates, the spectrum can be strictly orthogonally decoupled into the amplitude spectrum $A\left(u,v\right)$ and the phase spectrum $\mathrm{\Phi}\left(u,v\right)$:
\begin{equation}
	P_{ld}(u,v)=A(u,v)\cdot e^{j\Phi(u,v)}.
	\label{eq2}
\end{equation}

From the perspective of Fourier analysis, a well-established consensus dictates that the global topology, edge locations, and anatomical contours of an image are highly compressed and locked within the phase spectrum $\Phi$, while the intensity of energy in each frequency band and most random noise are primarily reflected in the amplitude spectrum $A$. Xian \emph{et al.} \cite{10021672} applied this property to medical segmentation. This fundamental property inspires us to maintain strict phase invariance while perturbing the amplitude. Such an approach guarantees that all generated pseudo-samples remain absolutely aligned anatomically and semantically, enabling the design of amplitude-based perturbation operators to effectively decouple spatial noise correlations.

Meanwhile, to construct the desired unbiased noisy sample bank, Parseval's theorem \cite{hughes1965physical} is introduced. This theorem establishes that the total energy of the signal in the data domain is conserved from its total energy in the frequency domain, described as:
\begin{equation}
	\sum_{h=0}^{H-1}\sum_{w=0}^{W-1}\left|p_{ld}(h,w)\right|^2 = \frac{1}{HW}\sum_{u=0}^{H-1}\sum_{v=0}^{W-1}\left|A(u,v)\right|^2,
	\label{eq3}
\end{equation}
this ensures that when the noise distribution in the high-frequency region is perturbed in the frequency domain, the generated target samples do not experience significant drift in global statistical properties or macroscopic contrast. By combining the frequency domain decoupling characteristics with the law of energy conservation, we can construct a sample bank with topological consistency and noise decoupling without the need for additional paired data.

In addition to the frequency-domain analysis described above, we conducted more in-depth experiments, as shown in Fig. \ref{fig2}. Based on the power-law attenuation characteristics in the spectrum, most clean information is concentrated in the low-frequency region $A_{low}$, while noise and other detail-degrading components are concentrated in the high-frequency region $A_{high}$. This finding inspired us to propose a novel self-supervised denoising strategy: Anchoring both the frequency-domain phase $\Phi$ and the low-frequency region $A_{low}$, while perturbing the amplitude of $A_{high}$. Consequently, this strategy minimizes damage to underlying anatomical semantics while effectively decoupling spatial noise correlations.
\subsection{Proposed FrequencyCT}
\label{Proposed FrequencyCT}
In this subsection, we first introduce two perturbation operators for sample generation (Sec.~\ref{Generation of Pseudo-sample}), then describe the training and inference processes (Sec.~\ref{Training and Inference}).
\subsubsection{Pseudo-sample Generation}
\label{Generation of Pseudo-sample}
\begin{figure}[!htbp]
	\centering
	\includegraphics[width=\linewidth]{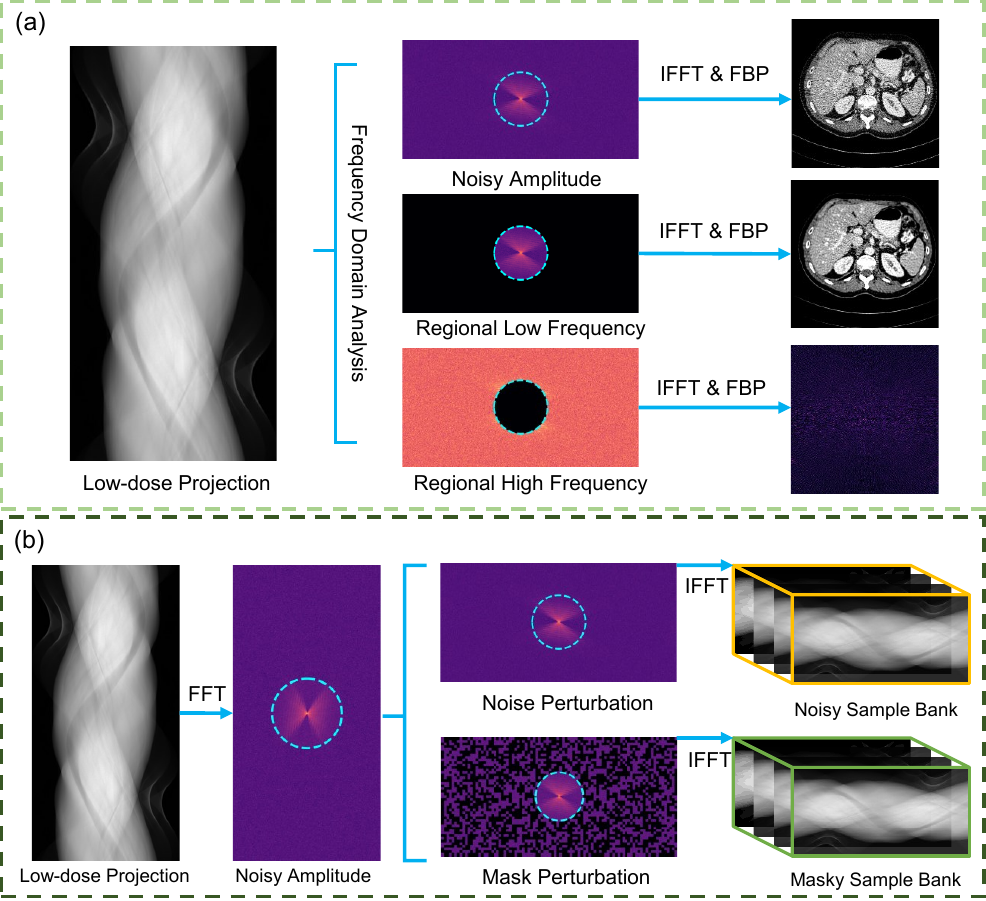}
	\caption{Motivation and description of frequency-domain pseudo-sample generation. (a) Illustration of noise distribution characteristics in the frequency domain. (b) A brief overview of pseudo-sample generation via high-frequency region perturbations outside the anchor circle. }
	\label{fig2}
\end{figure}
Based on the above motivation, to construct a self-supervised pseudo-sample bank that effectively decouples noise, this work defines the polar radius in polar coordinates as $\rho(u,v)=\sqrt{u^2+v^2}$, the region protection limit as $r_{min}$, and a dynamic perturbation limit $R_{rand}\sim \mathcal{U}(r_{1},r_{2})$ that follows a uniform distribution, satisfying $r_{min}=r_{1}<r_{2}$. Based on this, two perturbation operators are defined:

\textbf{Phase-Preserving Noise Perturbation (PPNP).} While ensuring the anchoring of regional low-frequency physical information $A_{low}^{\rho\leq R_{rand}}$, a centrosymmetric random perturbation matrix $Z$ is applied to the high-frequency region $A_{high}^{\rho>R_{rand}}$, satisfying $\mathbb{E}\left[Z\right]=1$. The generated samples can be described as:
\begin{equation}
	p_{ld}^{noise}=\mathcal{F}^{-1}\left(\left(A_{low}^{\rho\leq R_{rand}}+Z\odot A_{high}^{\rho>R_{rand}}\right)e^{j\Phi}\right),
	\label{eq4}
\end{equation}
where $\mathcal{F}^{-1}\left(\cdot\right)$ represents the inverse Fourier transform, and $p_{ld}^{noise}$ represents the noise-perturbated projection domain pseudo-sample. It maintains the existence of the high-frequency topology but applies random perturbations to the energy amplitude, reconstructing the noise distribution. Furthermore, if different perturbations are applied simultaneously, a massive number of samples $\mathcal{T}_{noise}=\{p_{ld}^{noise_{1}},\ldots,p_{ld}^{noise_{n}}\}$ can be generated in a very short time, weakening the correlation between noises.

\textbf{Phase-Preserving Mask Perturbation (PPMP).} Similarly, while ensuring the regional low-frequency anchorage $A_{low}^{\rho\le r_{min}}$, a centrosymmetric Bernoulli-compliant binary mask $B\sim\mathrm{Bernoulli}(\beta)$ is applied to the high-frequency region $A_{high}^{\rho>r_{min}}$. The generated samples can be described as:
\begin{equation}
	p_{ld}^{mask}=\mathcal{F}^{-1}\left(\left(A_{low}^{\rho\leq r_{min}}+B\odot A_{high}^{\rho>r_{min}}\right)e^{j\Phi}\right),
	\label{eq5}
\end{equation}
where $\beta$ represents the regional high-frequency retention probability. Based on masking operations, high-frequency noisy data can be randomly truncated. After transformation back to the projection domain, long-range coherent streak noise is substantially disrupted, and the spatial correlation of noise can be broken. Similarly, a masky sample bank $\mathcal{T}_{mask}=\{p_{ld}^{mask_{1}},\ldots,p_{ld}^{mask_{n}}\}$ can be generated highly efficiently.

\begin{figure}[!htbp]
	\centering
	\includegraphics[width=\linewidth]{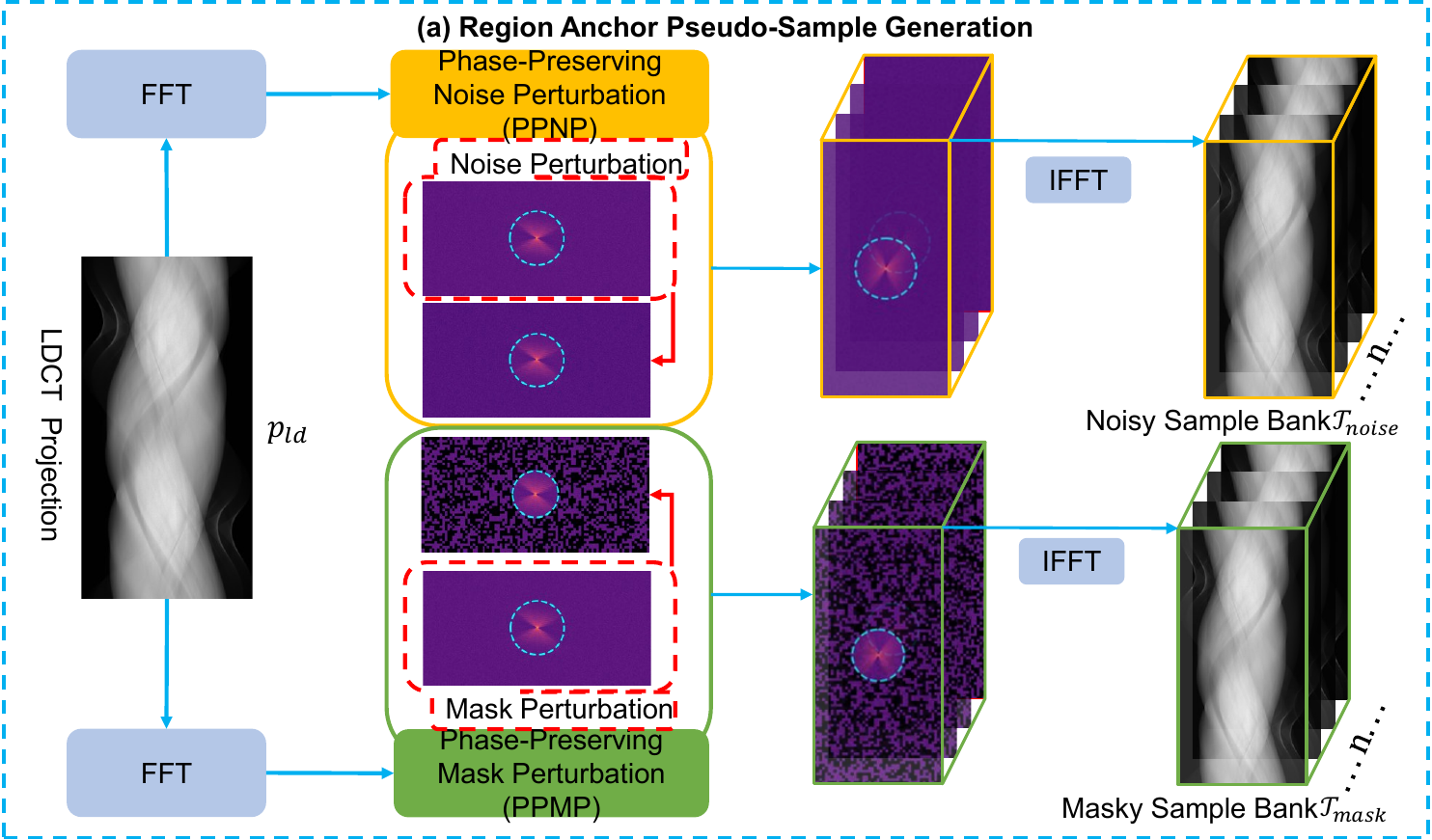}
	\caption{Overview of generating pseudo-samples. From left to right: After obtaining the LDCT projection $p_{ld}$, perform a Fourier transform on it to the frequency domain, and apply PPNP and PPMP at the same time to generate noisy sample bank $\mathcal{T}_{noise}$ and masky sample bank $\mathcal{T}_{mask}$.}
	\label{fig1_generation}
\end{figure}
By applying distinct perturbations to maximally isolated noise, both operators effectively decouple spatial noise correlations while preserving identical macroscopic anatomical information. A more detailed description of pseudo-sample generation is shown in Fig. \ref{fig1_generation}.
\subsubsection{Training and Inference}
\label{Training and Inference}
As described above, after obtaining pseudo-samples generated by different perturbation operators, they are randomly selected as the network training input $p_{ld}^{mask}$ and the target $p_{ld}^{noise}$, respectively. Previous studies \cite{ma2025pixel2pixel,Unal2024Proj2ProjSL} have defined the $l_2$ norm objective function $\mathbb{E}\parallel \cdot \parallel_2^2$ used to train the denoising network $f_\theta(\cdot)$ as follows:
\begin{equation}
	\mathbb{E}\parallel f_\theta\left(p_{ld}^{mask}\right)-p_{ld}^{noise}\parallel_2^2,
	\label{eq6}
\end{equation}
however, directly optimizing Eq. \eqref{eq6} can lead to unstable gradients. Dictated by the underlying physical mechanisms of CT, regions with higher true projection values are rendered highly susceptible to statistical fluctuations. Such signal-dependent extreme outliers can severely disrupt the training process. To mitigate this and stabilize the training process, we introduce a data truncation mechanism. Specifically, the generated pseudo-samples are clamped before training:
\begin{equation}
	p_{ld_{c}}^{mask}=\mathrm{clamp}\left(p_{ld}^{mask},0,T\right),p_{ld_{c}}^{noise}=\mathrm{clamp}\left(p_{ld}^{noise},0,T\right),
	\label{eq7}
\end{equation}
where $\mathrm{clamp}(\cdot)$ bounds pseudo-samples within $[0, T]$. Setting $T=1$, the robust objective for FrequencyCT is refined as:
\begin{equation}
	\mathbb{E}\parallel f_\theta\left(p_{ld_c}^{mask}\right)-p_{ld_c}^{noise}\parallel_2^2.
	\label{eq8}
\end{equation}

To parameterize the denoising network $f_\theta(\cdot)$, we employ a compact convolutional neural network (CNN) architecture. Specifically, the network is formulated as a 5-layer CNN, where each layer consists of a $3 \times 3$ convolution followed by a ReLU activation function. 

The optimal convolutional network $f_{\theta^*}(\cdot)$ is obtained by training via Eq. \eqref{eq8}. During the inference phase, the full-range LDCT projection $p_{ld}$ is directly fed into the optimized network $f_{\theta^*}(\cdot)$ to yield the denoised data $p$. Finally, the denoised reconstructed image $x_0$ is obtained by $\mathrm{FBP}(p)$. A description of truncated training and inference is shown in Fig. \ref{fig1_train}. The overall workflow of FrequencyCT is illustrated in Algorithm \ref{alg:A1}.
\begin{figure}[!htbp]
	\centering
	\includegraphics[width=\linewidth]{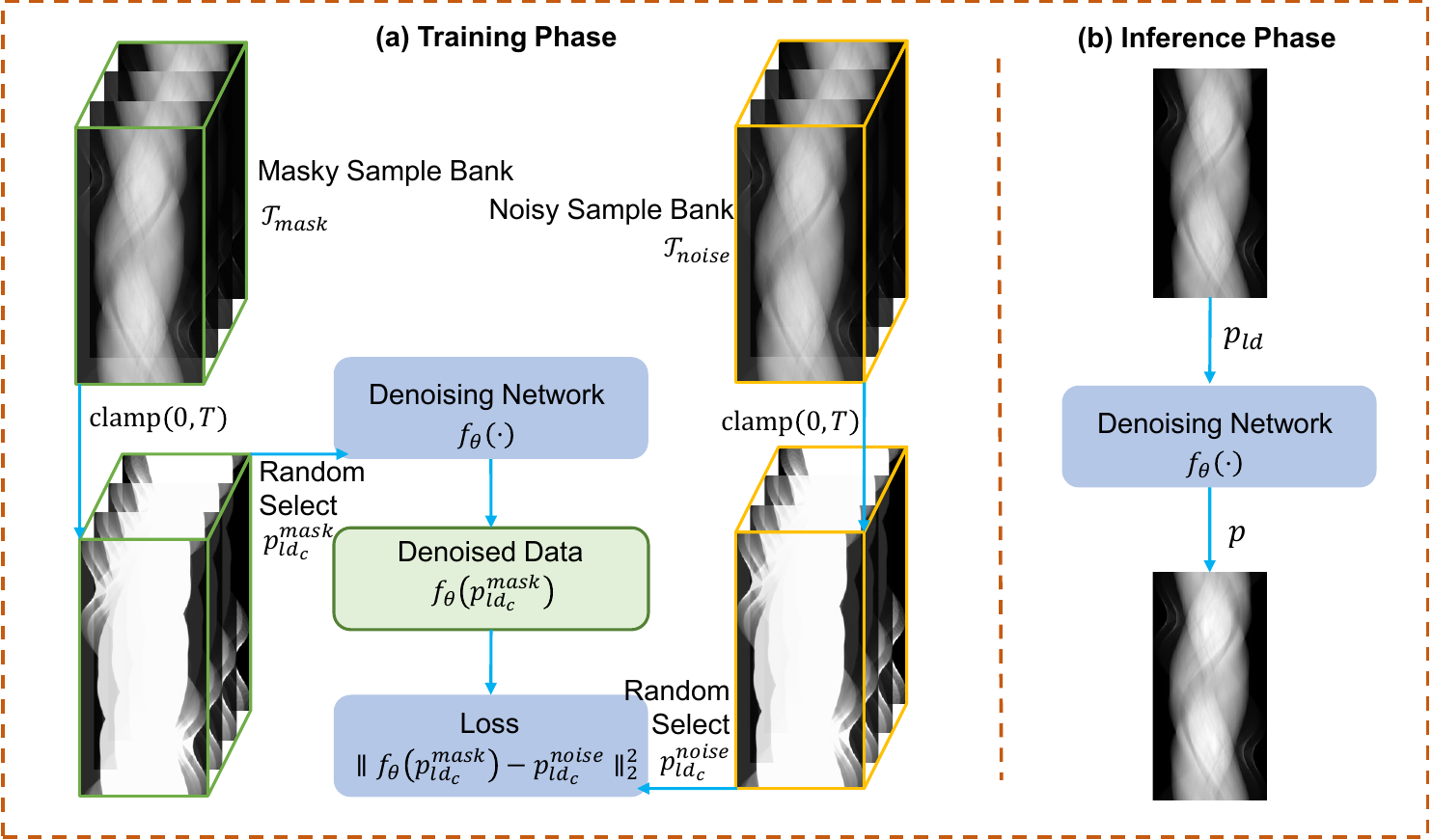}
	\caption{Network training and inference. (a) Training phase: Before training, the generated samples are truncated via Eq. \eqref{eq7}, and $p_{ld_c}^{noise}$ and $p_{ld_c}^{mask}$ are randomly selected for training. (b) Inference phase: The trained network optimizes the original noisy data.}
	\label{fig1_train}
\end{figure}
\begin{algorithm}[!htbp]
	\caption{Training and Inference Process of FrequencyCT}
	\label{alg:A1}
	%\hrule
	%\vspace{2mm}
	
	%\color{blue} % 如果需要标蓝请保留此行，否则可以直接删除
	
	%\hangindent %负责悬挂缩进，\hangafter=1 表示从第二行开始缩进
	\noindent\hangindent=3.5em\hangafter=1
	\textbf{Input:} Single LDCT projection data $p_{ld}$, dynamic perturbation range $R_{rand}$, number of samples $n$ \par
	
	\noindent\hangindent=4.2em\hangafter=1
	\textbf{Output:} Denoised data $p$ and image $x_0$ \par
	
	\noindent \textbf{Training Phase:} \\
	1: Generate samples $p_{ld}^{noise}$ and $p_{ld}^{mask}$ via Eqs. \eqref{eq4} and \eqref{eq5}\\
	2:  Truncate samples to obtain $p_{ld_c}^{noise}$ and $p_{ld_c}^{mask}$ via Eq. \eqref{eq7}\\
	3:  \textbf{Repeat} \\
	4:  \quad Update $f_\theta(\cdot)$ by Eq. \eqref{eq8} \\
	5:  \textbf{Until} convergence to $f_{\theta^*}(\cdot)$
	
	\noindent \textbf{Inference Phase:} \\
	6:  $p \leftarrow f_{\theta^*}(p_{ld})$ \\
	7:  $x_0 = \mathrm{FBP}(p)$
	%\vspace{2mm}
	
	%\hrule
\end{algorithm}
\subsection{Theoretical Analysis}
\label{Theoretical Analysis}
\begin{figure*}[!htbp]
	\centering
	\includegraphics[width=\linewidth]{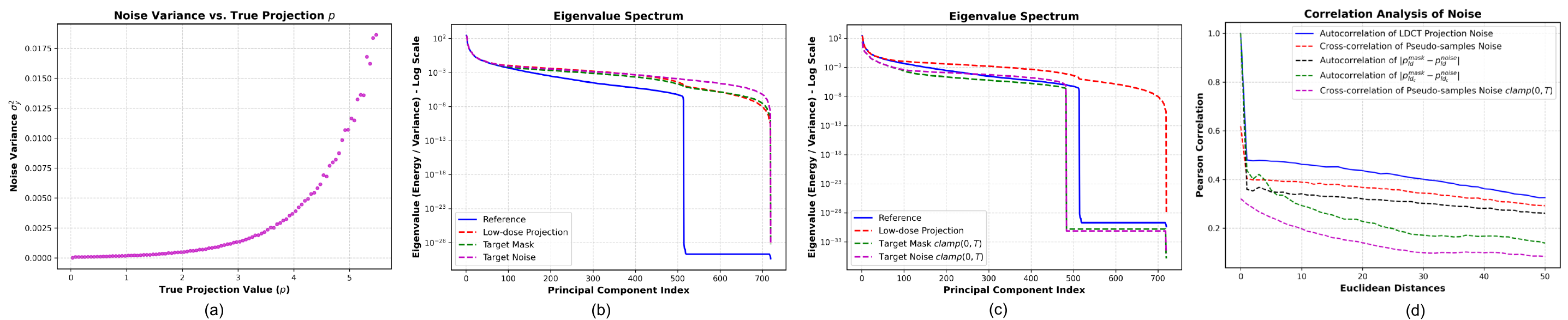}
	\caption{Experimental verification of the relationship between the noise variance of noisy projections and the true projection values, alongside an analysis of PCA and noise correlations before and after truncation. (a) Experimental validation of Eq. \eqref{eq9} via LDCT projection and clean data. (b) PCA of pseudo-samples generated by different perturbation operators before truncation, along with source noisy and reference data. (c) PCA of pseudo-samples generated by different perturbation operators after truncation, along with source noisy and reference data. (d) Analysis of the impact of pseudo-samples and truncation on noise correlation.}
	\label{fig3}
\end{figure*}

In this subsection, we use a theorem to explain why the truncation operation in Eq. \eqref{eq7} enables convergence in network training and denoising across the entire numerical range during inference.
\begin{theorem}
	\label{theorem}
	In LDCT projection domain, assuming the number of photons received by the $i$-th detector bin follows a Poisson distribution with an incident photon intensity $I_0$, let $I_i$ denote the transmitted photon count penetrating the object, $y_i$ denote the noisy observation value with variance $\sigma_{y_i}^2$, and $p_i$ denote the corresponding clean projection value. then there holds:
	\begin{equation}
		\sigma_{y_i}^2\approx\frac{\exp(p_{i})}{I_0},
		\label{eq9}
	\end{equation}
the above formulation indicates that under noisy conditions, the noise variance and the true projection value at each individual detector element are exponentially positively correlated.
\end{theorem}

\begin{proof}
	In CT projection domain, the Beer-Lambert law \cite{swinehart1962beer} defines the transmitted intensity as $I_i=I_0\exp(-p_i)$. Driven by the quantum nature of the photon-counting process, the random variable $N_i$ recorded by the $i$-th detector element operates as a Poisson-distributed variable, satisfying $N_i \sim \text{Poisson}(I_i)$. The inherent property of the Poisson distribution dictates that its expectation strictly equals its variance, then:
	\begin{equation}
		\mathbb{E}[N_i]=\mathbb{V}[N_i]=I_i=I_0\exp(-p_i).
		\label{eq10}
	\end{equation}
	By applying Eq. \eqref{eq10} inversely, under noisy measurement conditions, $y_i$ can be defined as a nonlinear mapping function with respect to $N_i$:
	\begin{equation}
		y_i=f(N_i)=-\ln\left(\frac{N_i}{I_0}\right)=-\ln(N_i)+\ln(I_0).
		\label{eq11}
	\end{equation}
	Due to the nonlinearity of $f(N_i)$, the Taylor formula \cite{odibat2007generalized} is introduced and expanded at the mean value $\mu_i=\mathbb{E}[N_i]$:
	\begin{equation}
		f(N_i)\approx f(\mu_i)+f^{\prime}(\mu_i)(N_i-\mu_i).
		\label{eq12}
	\end{equation}
	Applying the constant scaling property, the variance of both sides is given:
	\begin{equation}
		\begin{aligned}
			\mathbb{V}[f(N_i)] &\approx \mathbb{V}[f(\mu_i)+f^{\prime}(\mu_i)(N_i-\mu_i)] \\
			&\approx \mathbb{V}[f^{\prime}(\mu_i)(N_i-\mu_i)] \\
			&\approx [f^{\prime}(\mu_i)]^{2}\cdot\mathbb{V}[N_i-\mu_i]\\
			&\approx [f^{\prime}(\mu_i)]^{2}\cdot\mathbb{V}[N_i].
		\end{aligned}
		\label{eq13}
	\end{equation}
	According to $f^\prime\left(N_i\right)=-\frac{1}{N_i}$, $\mu_i=\mathbb{E}\left[N_i\right]$, and substituting Eq. \eqref{eq10} and \eqref{eq11} into the above equation, it yields:
	\begin{equation}
		\begin{aligned}
			\mathbb{V}[y_i]  \approx\sigma_{y_i}^{2}\approx\left(-\frac{1}{I_{0}\mathrm{exp}\left(-p_i\right)}\right)^{2}\cdot\left(I_{0}\mathrm{exp}\left(-p_i\right)\right) 
			\approx\frac{\exp(p_i)}{I_0}.
		\end{aligned}
		\label{eq14}
	\end{equation}
	This completes the proof, and Fig. \ref{fig3} (a) have also been experimentally verified.
\end{proof}

As described by Eq. \eqref{eq9}, the variance of noise increases exponentially with the increase of the true projection value $p_i$, meaning that the higher $p_i$, the stronger the data volatility and the lower local $\mathrm{SNR}$. For a specific detector element $i$, the local $\mathrm{SNR}$ is defined as:
\begin{equation}
	\mathrm{SNR}_i=\frac{p_i}{\sqrt{\mathbb{V}[y_i]}},
	\label{eq15}
\end{equation}
substituting Eq. \eqref{eq14} into the above equation, it yields:
\begin{equation}
	\mathrm{SNR}_i=\frac{p_i}{\sqrt{\frac{\exp{(p_i)}}{I_0}}}=p_i\sqrt{I_0}\exp\left(-\frac{p_i}{2}\right).
	\label{eq16}
\end{equation}
Analysis of the above equation shows that $p_i=2$ is the extreme point of $\mathrm{SNR}_i$, and subsequently, as $p_i\rightarrow\infty$, $\mathrm{SNR}_i\rightarrow0$. This implies that regions with larger true projection values are inevitably dominated by massive random noise, where authentic anatomical signals are gradually buried. According to Neighbor2Neighbor \cite{huang2021neighbor2neighbor}, the following lemma holds:
\begin{equation}
	\begin{aligned}
		\mathbb{E}_{p,y}\parallel f_\theta(y)-p\parallel_2^2 =\mathbb{E}_{p,y,z}\parallel f_\theta(y)&-z\parallel_2^2-\sigma_z^2 \\
		& +2\boldsymbol{\varepsilon}\mathbb{E}_{p,y}(f_\theta(y)-p),
	\end{aligned}
	\label{eq17}
\end{equation}
where ${\boldsymbol{\varepsilon}}$ represents the difference in the underlying true signal between different noisy observations $z$ and $y$. When ${\boldsymbol{\varepsilon}} \rightarrow 0$, the self-supervised target is equivalent to supervised learning using clean data $p$. However, driven by Theorem \ref{theorem}, the true signals in high-value projection regions are masked by highly fluctuating non-stationary noise, which disrupts the optimization gradient of the network parameters $\theta$, even forcing the network to overfit the noise to generate an identity mapping. Therefore, this work uses Eq. \eqref{eq7} to truncate the data to isolate highly fluctuating noisy data, so that steering the network optimization toward regions with high signal fidelity to capture stable gradients. Furthermore, the principal component analysis (PCA) \cite{mackiewicz1993principal} in Fig. \ref{fig3} (b) and (c) demonstrates that samples from different perturbation operators are orthogonal, effectively breaking noise correlations. As additionally evidenced by Fig. \ref{fig3} (d), data truncation substantially decouples noise correlations, yielding the optimal training operator $f_{\theta^\ast}(\cdot)$.
\begin{table*}[!htbp]
	\centering
	\caption{Quantitative results for different datasets and doses (mean ± standard deviation). The best results are highlighted in bold.}
	\label{table1}
	\resizebox{\linewidth}{!}{%
		\begin{tabular}{cccc|ccc|ccc} 
			\toprule
			\multirow{2}{*}{Method} & \multicolumn{3}{c|}{Mayo2016 Dataset - 25\% Dose}               & \multicolumn{3}{c|}{Mayo2016 Dataset - 10\% Dose}                                      & \multicolumn{3}{c}{LIDC-IDRI Dataset - 25\% Dose}                 \\ 
			\cline{2-10}
			& PSNR(dB)~$\uparrow$ & SSIM(\%)~$\uparrow$ & RMSE~$\downarrow$   & PSNR(dB)~$\uparrow$ & SSIM(\%)~$\uparrow$ & RMSE~$\downarrow$                          & PSNR(dB)~$\uparrow$  & SSIM(\%)~$\uparrow$ & RMSE~$\downarrow$    \\ 
			\hline
			FBP                      & 37.76±2.27          & 89.72±5.24          & 26.78±7.31          & 34.66±2.52          & 80.69±8.96          & 38.62±11.73                                & 36.99±3.49           & 81.05±10.79         & 30.63±12.32          \\
			BM3D                     & 40.67±1.03          & 96.36±1.52          & 18.68±2.17          & 39.65±1.64          & 94.48±2.68          & 21.16±4.54                                 & 40.78±1.69           & 89.95±4.30          & 18.54±4.35           \\
			B2U                      & 39.85±1.89          & 94.58±2.58          & 20.84±4.64          & 36.40±2.45          & 86.44±6.83          & 31.53±9.50                                 & 40.03±2.87           & 88.21±6.29          & 21.08±7.55           \\
			Neighbor2Neighbor        & 40.12±1.87          & 94.68±2.60          & 20.20±4.54          & 38.29±2.13          & 91.32±4.53          & 25.12±6.55                                 & 39.48±3.39           & 86.57±7.96          & 22.92±9.29           \\
			Noisier2Noise             & 39.64±0.74          & 95.47±0.61          & 20.91±1.76          & 39.27±0.89          & 94.36±1.01          & 21.86±2.30                                 & 38.79±1.08           & 92.46±1.36          & 23.17±3.42           \\
			Noise2Sim                & 40.31±0.65          & 95.44±0.37          & 19.33±1.43          & 35.46±1.28          & 86.09±3.36          & 34.08±5.09                                 & 37.19±0.55           & 87.57±1.89          & 27.96±1.85           \\
			Noise2detail             & 39.78±1.64          & 94.57±2.16          & 21.24±4.09          & 37.98±1.95          & 91.36±4.27          & 25.92±6.19                                 & 39.85±2.50           & 88.18±5.79          & 21.24±6.59           \\
			ZS-N2N                   & 39.66±1.74          & 94.09±2.56          & 22.04±4.41          & 37.69±1.93          & 89.93±4.56          & 26.79±6.38                                 & 39.29±2.67           & 86.61±6.59          & 22.78±7.33           \\
			Pixel2Pixel              & 40.83±1.40          & 96.18±1.63          & 18.60±3.06          & 39.29±1.92          & 93.05±3.71          & 22.26±5.30                                 & 40.89±2.39           & 89.41±5.54          & 18.77±5.90           \\
			FrequencyCT              & \textbf{40.94±1.49} & \textbf{96.74±1.13} & \textbf{18.19±3.22} & \textbf{39.77±1.70} & \textbf{95.35±1.73} & \textbf{20.91±3.57}                        & \textbf{40.99±1.93}  & \textbf{90.32±3.63} & \textbf{18.29±4.18}  \\ 
			\cmidrule{1-10}
			\multirow{2}{*}{Method} & \multicolumn{3}{c|}{Mayo2020 Dataset - Ultra Low Dose}           & \multicolumn{3}{c|}{\textcolor[rgb]{0.125,0.122,0.125}{CTSpine1K}~Dataset - 10\% Dose} & \multicolumn{3}{c}{LIDC-IDRI Dataset - 10\% Dose}                  \\ 
			\cline{2-10}
			& PSNR(dB)~$\uparrow$ & SSIM(\%)~$\uparrow$ & RMSE~$\downarrow$   & PSNR(dB)~$\uparrow$ & SSIM(\%)~$\uparrow$ & RMSE~$\downarrow$                          & PSNR(dB)~$\uparrow$~ & SSIM(\%)~$\uparrow$ & RMSE~$\downarrow$    \\ 
			\hline
			FBP                      & 26.28±2.14          & 55.62±6.83          & 100.24±28.01        & 30.39±2.10          & 60.62±9.29          & 62.22±15.08                                & 33.10±4.29           & 68.18±16.01         & 50.31±27.96          \\
			BM3D                     & 29.12±1.45          & 69.92±6.76          & 65.34±20.61         & 35.22±3.54          & 77.24±6.40          & 38.01±17.97                                & 38.20±4.50           & 83.33±10.36         & 28.75±24.26          \\
			B2U                      & 25.49±1.36          & 48.10±6.28          & 107.48±16.86        & 32.84±2.37          & 70.63±7.06          & 47.36±14.17                                & 35.88±4.22           & 77.48±12.12         & 36.81±23.50          \\
			Neighbor2Neighbor        & 31.83±2.65          & 76.43±5.12          & 54.17±22.25         & 32.78±2.34          & 68.76±7.74          & 47.62±13.59                                & 36.45±4.63           & 77.57±12.78         & 35.18±23.26          \\
			Noisier2Noise             & 30.40±3.29          & 72.52±6.82          & 65.40±30.61         & 35.94±3.50          & 85.83±5.34          & 34.92±17.02                                & 28.80±3.86           & 82.81±5.01          & 79.40±30.81          \\
			Noise2Sim                & 32.80±1.96          & 76.64±4.39          & 47.13±13.34         & 37.85±1.17          & 85.79±4.19          & 25.83±3.65                                 & 33.30±2.43           & 77.69±13.16         & 45.27±16.66          \\
			Noise2detail             & 31.10±1.66          & 73.79±4.57          & 64.53±12.23         & 33.71±2.50          & 72.67±6.26          & 43.13±13.93                                & 36.94±3.76           & 80.22±10.32         & 31.91±18.82          \\
			ZS-N2N                   & 29.56±1.54          & 73.78±4.87          & 66.34±11.69         & 33.17±2.36          & 70.57±5.96          & 45.67±13.92                                & 36.35±3.66           & 78.23±10.55         & 33.77±19.08          \\
			Pixel2Pixel              & 30.94±3.22          & 73.09±6.15          & 61.56±30.94         & 34.67±2.92          & 74.55±6.03          & 39.24±15.12                                & 37.93±4.13           & 81.29±10.74         & 29.09±19.74          \\
			FrequencyCT              & \textbf{33.01±1.50} & \textbf{76.56±4.89} & \textbf{45.36±8.21} & \textbf{38.59±0.95} & \textbf{86.59±1.78} & \textbf{23.65±2.57}                        & \textbf{39.39±1.50}  & \textbf{88.36±2.72} & \textbf{21.81±4.43}  \\
			\bottomrule
		\end{tabular}
	}
\end{table*}
\begin{figure*}[!htbp]
	\centering
	\includegraphics[width=\linewidth]{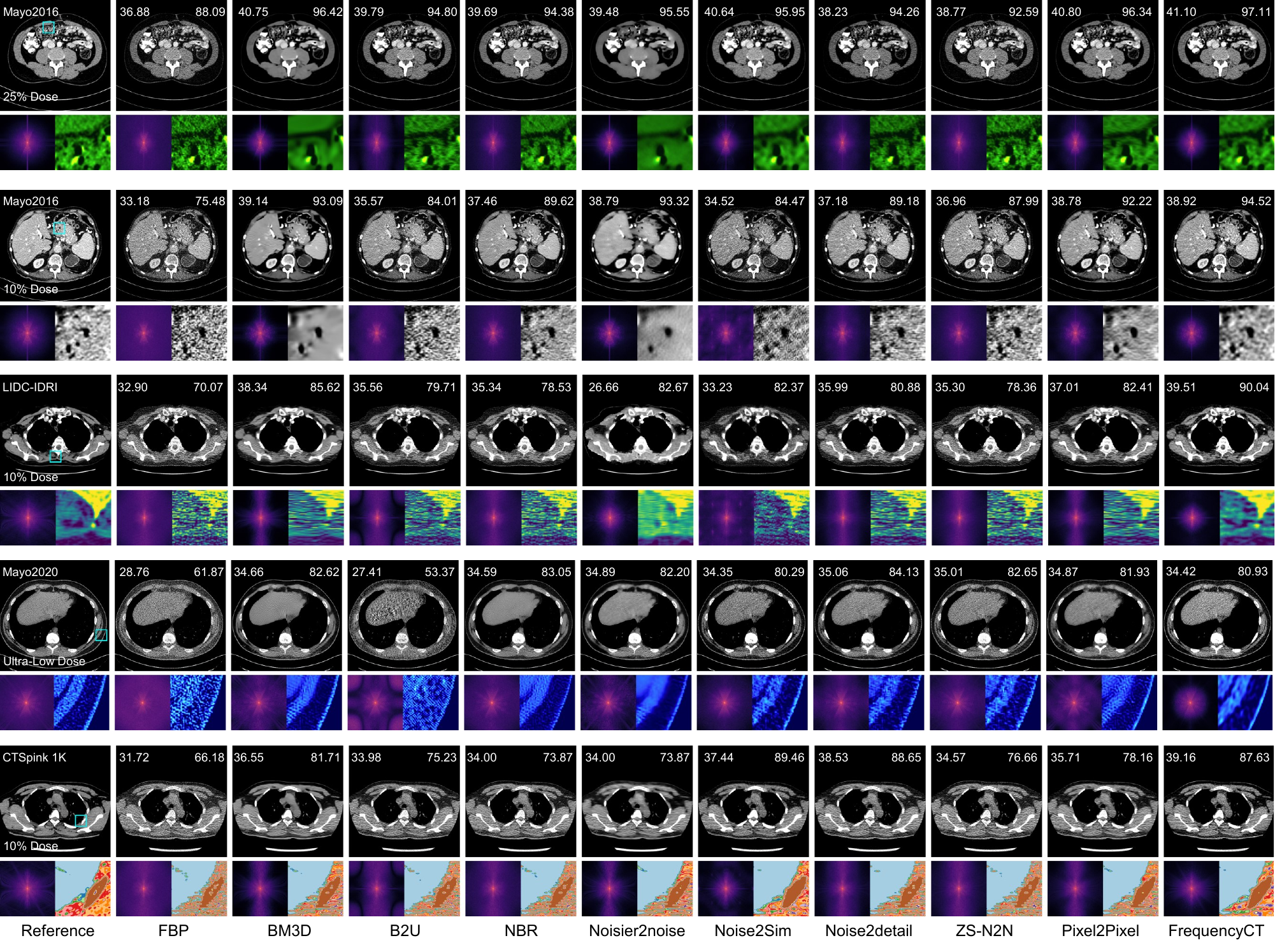}
	\caption{Quantitative results of denoising from different doses and simulation data. The blue boxes in the figure represent the regions of interest for each image, and the lower-left corner shows the performance of the denoised images in the frequency domain.}
	\label{fig6}
\end{figure*}

During inference, driven by $f_{\theta^\ast}\left(\cdot\right)$, it can generalize to the full-range projection domain, achieving high-quality denoising. For a detailed analysis, please refer to the Appendix.
\section{Experiments}
\label{Experiments}
In this section, we begin by introducing the dataset used (Sec. \ref{Datasets}), along with implementation details, comparison methods, and evaluation metrics (Sec. \ref{Implementation Details}). We then present results on simulated and real data (Sec. \ref{Experimental results}), and conclude with an overview of the ablation study conducted on our method (Sec. \ref{Ablation Studies}).
\begin{figure*}[!htbp]
	\centering
	\includegraphics[width=\linewidth]{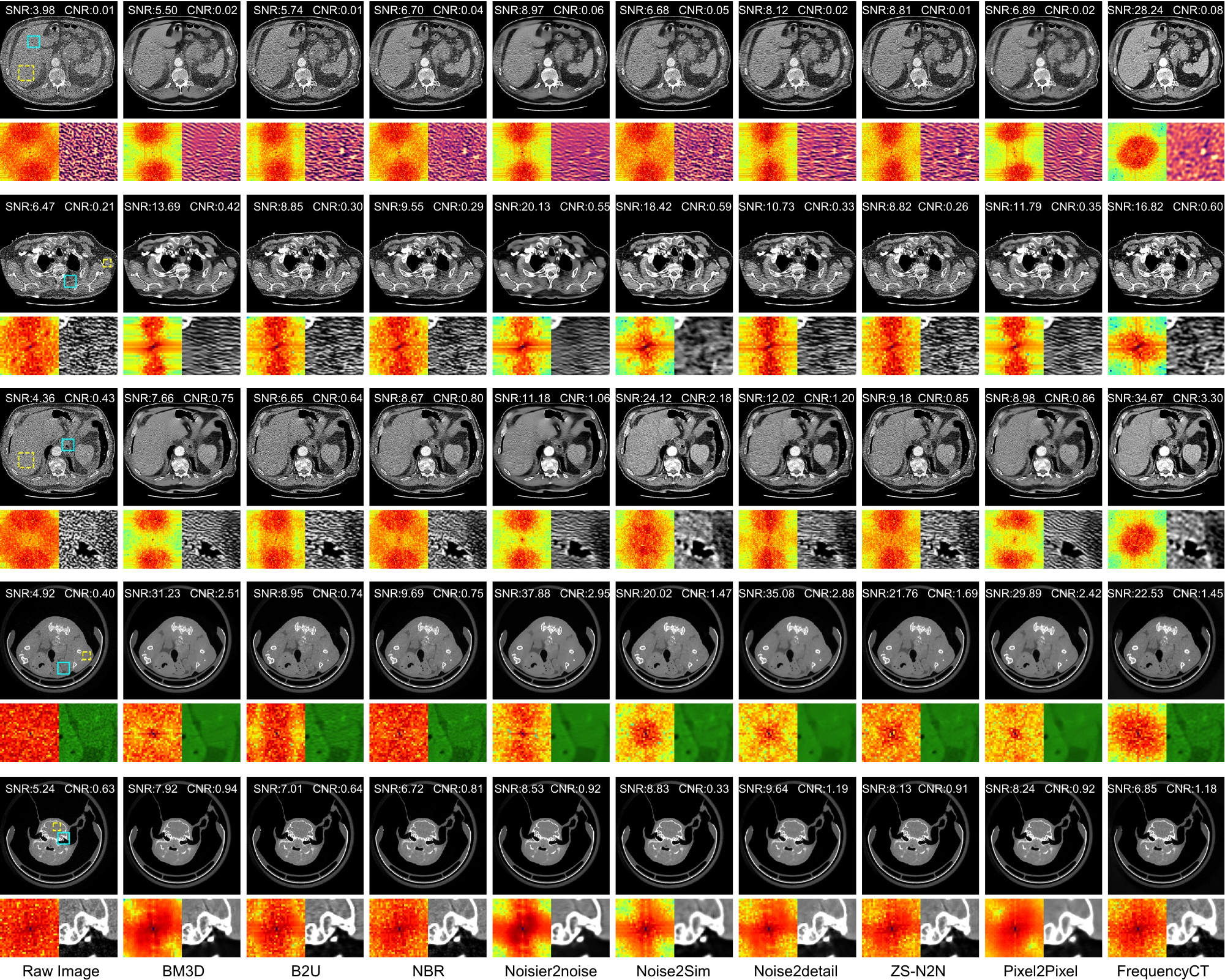}
	\caption{Comparison of real-world human and mouse data with different methods. SNR and CNR were calculated for each experimental result, and the yellow dashed line represents the background region used in the calculation, and the blue line represents the region of interest. Additionally, the lower left corner of each image displays NPS to demonstrate the superiority of FrequencyCT.}
	\label{fig5}
\end{figure*}
\subsection{Experimental Settings}
\label{Experimental Settings}
\subsubsection{Datasets}
\label{Datasets}
To fully demonstrate the effectiveness of the method, this study selected a large amount of simulated data and collected real data. For the simulated data, we selected Mayo2016 \cite{chen2016open}, LIDC-IDRI \cite{armato2011lung}, CTSpine1k \cite{deng2021ctspine1k}, and Mayo2020 \cite{moen2021low} data, with Mayo2020 being used as the ultra-low dose. We used Torch-radon \cite{ronchetti2020torchradon} and ODL \cite{adler2017operator} tools to simulate low-dose CT data at 25\% and 10\% doses, with the simulation method referenced \cite{xie2017robust}. The projection domain contained 1440 angles and 720 detector pixels. For Mayo2020, no simulation was performed; instead, publicly available real noisy data was used as the ultra-low dose. For the real data, with 0.2 mA as the standard dose, we collected dual-energy data from mice at 50\% and 25\% doses, with energy ranges of 20-31 keV and 31-70 keV, respectively. In addition, low-dose data from several real patients were randomly selected from the LIDC-IDRI data for clinical research.
\subsubsection{Implementation Details}
\label{Implementation Details}
The experiments of this work were implemented using the PyTorch framework on an NVIDIA RTX 4090D GPU (24GB). During training, the Adam optimizer was used for 1500 epochs with an initial learning rate of 0.001. For the hyperparameters of pseudo-sample generation, we set $n=4$ and $\left(r_1,r_2\right)=(0.5,0.6)$.

To comprehensively evaluate our method, we construct a comparative baseline encompassing traditional, pre-trained self-supervised, and zero-shot denoising approaches. The selected baseline methods include BM3D \cite{dabov2007image}, Blind2Unblind \cite{wang2022blind2unblind}, denoted as B2U, Neighbor2Neighbor \cite{huang2021neighbor2neighbor}, denoted as NBR, Noisier2Noise \cite{moran2020noisier2noise}, Noise2Sim \cite{niu2022noise}, Noise2detail \cite{chobola2025lightweight}, ZS-N2N \cite{mansour2023zero}, and Pixel2Pixel\cite{ma2025pixel2pixel}. For methods requiring pre-training, the datasets are partitioned by allocating eight Mayo2016 patients for training and two for testing, alongside 49 random training cases from LIDC-IDRI and 38 from CTSpine 1K. Regarding the quantitative evaluation metrics, the assessment of simulated data employs Peak Signal-to-Noise Ratio (PSNR), Structural Similarity Index Measure (SSIM), and Root Mean Square Error (RMSE). Conversely, the evaluation of real clinical data relies on Signal-to-Noise Ratio (SNR), Contrast-to-Noise Ratio (CNR), and Noise Power Spectrum (NPS) due to the inherent absence of reference labels. All quantitative metrics are rigorously computed within a predefined numerical interval of $[-1024, 3072]$.
\subsection{Experimental Results}
\label{Experimental results}
\subsubsection{Results on Simulation Data}
To fully verify the leading advantages of the frequency domain pseudo-sample generation mechanism and projection data truncation strategy in the field of denoising, we constructed a large-scale simulation dataset for comprehensive comparative evaluation. Although this study mainly relies on the Poisson distribution assumption in the method derivation stage, we deliberately introduced a more complex and physically realistic Poisson and Gaussian mixture noise model in the simulation experiments to test the generalization robustness of the model. As demonstrated by the quantitative metrics in Table \ref{table1} and the visual comparisons in Fig. \ref{fig6}, FrequencyCT exhibits distinctly superior performance across all evaluation dimensions. Furthermore, we observed that as the radiation dose decreases, mainstream methods, including B2U, NBR, Noisier2Noise, and Pixel2Pixel, suffer varying degrees of performance degradation. This degradation is intuitively manifested as excessive smoothing of local textures, blurring of anatomical structure edges, incomplete removal of complex noise, and even the emergence of new noise morphologies. Fundamentally, this occurs because escalating noise levels intensify spatial noise correlations. These strong correlations destroy the independence of local features, rendering the optimization gradients severely unstable and oscillatory during network training.
\subsubsection{Result on Real-world Data}
Current denoising research largely relies on simulated data to evaluate method performance. However, noise in real-world physical imaging environments is extremely complex, often exhibiting a mixture of Poisson and Gaussian distributions, and even more complex non-stationary statistical characteristics. This poses a significant challenge to the generalization ability of existing methods on real data. As shown in Fig. \ref{fig5}, FrequencyCT significantly outperforms other baselines in both visual perception and quantification metrics on real human data. Specifically, while traditional methods like BM3D possess some noise suppression capabilities, they inevitably result in severe loss of texture details. Among models that rely on additional data for pre-training, Noise2Sim, while leading in performance, still causes edge blurring, and its network training time is high, with performance dependent on the distribution of neighboring batches. In contrast, FrequencyCT achieves superior image reconstruction quality without any external prior data. Furthermore, for mouse dual-energy CT data, the fidelity of soft tissue details and structural edges is particularly critical. Some existing methods often sacrifice signal-to-noise ratio by excessively smoothing image details. For example, the Pixel2Pixel model, based on the assumption of local similarity, is prone to causing excessive smoothing of anatomical structures when dealing with through-line noise with strong spatial correlation. FrequencyCT, on the other hand, successfully avoids the above-mentioned defects, achieving a balance between eliminating complex background noise and preserving realistic tissue features. In contrast, FrequencyCT demonstrates superior anti-interference capabilities. During the generation of pseudo-samples, strict phase preservation and orthogonal-like perturbation in high-frequency regions effectively eliminate spatial correlations of noise. Simultaneously, consistent data truncation rigorously isolates extreme statistical outliers, significantly alleviating the burden of gradient updates during network optimization. This allows the model to penetrate dense noise fog, accurately perceive and reconstruct anatomical features, thus maintaining extremely high fidelity even under extreme imaging conditions.
\subsection{Ablation Study}
\label{Ablation Studies}
In this subsection, we first introduce the comparison of the objective function before and after truncation (Sec. \ref{E||}), and then introduce the hyperparameters and computation time (Sec. \ref{Hyperparameters and computation time}).
\subsubsection{$\mathbb{E}\parallel f_\theta(p_{ld}^{mask})-p_{ld}^{noise}\parallel_2^2~\mathrm{vs.}~\mathbb{E}\parallel f_\theta(p_{ld_c}^{mask})-p_{ld_c}^{noise}\parallel_2^2$}
\label{E||}
This work does not deny the effectiveness of research that uses full-range values such as $\mathbb{E}\parallel f_\theta\left(p_{ld}^{mask}\right)-p_{ld}^{noise}\parallel_2^2$ as objective functions for denoising. Rather, it proposes a novel paradigm to highlight that truncating both the input and the target in LDCT projection domain facilitates network optimization. Therefore, to demonstrate the superiority of truncated training, validation was performed using multiple simulated datasets with a 10\% dose. As shown in Tables \ref{table1} and \ref{table2}, even when using the objective function in Eq. \eqref{eq6}, our method outperforms most comparison methods. Furthermore, Table \ref{table2} clearly demonstrates that via Eq. \eqref{eq8} results in an average increase of 1.8 dB in PSNR, a 4.9\% increase in SSIM, and a 7 decrease in RMSE. This not only illustrates the effectiveness of generating pseudo-samples through noise and mask perturbations for denoising but also highlights how truncation training elevates network optimization to a new level.
\begin{table}[!htbp]
	\centering
	\caption{Quantitative evaluation of the performance of different objective functions across different datasets.}
	\label{table2}
	\resizebox{\linewidth}{!}{%
		\begin{tabular}{cccc|ccc} 
			\toprule
			\multirow{2}{*}{Dataset} & \multicolumn{3}{c|}{$\mathbb{E}\parallel f_\theta(p_{ld}^{mask})-p_{ld}^{noise}\parallel_2^2$} & \multicolumn{3}{c}{$\mathbb{E}\parallel f_\theta(p_{ld_c}^{mask})-p_{ld_c}^{noise}\parallel_2^2$}  \\ 
			\cline{2-7}
			& PSNR(dB)~$\uparrow$ & SSIM(\%)~$\uparrow$ & RMSE~$\downarrow$                                   & PSNR(dB)~$\uparrow$  & SSIM(\%)~$\uparrow$  & RMSE~$\downarrow$                                     \\ 
			\hline
			Mayo2016                  & 38.91±2.28         & 93.52±5.45         & 23.56±7.41                                         & \textbf{39.77±1.70} & \textbf{95.35±1.73} & \textbf{20.91±3.57}                                  \\
			LIDC-IDRI                 & 37.47±4.01         & 82.56±11.33        & 30.95±23.58                                        & \textbf{39.39±1.50} & \textbf{88.36±2.72} & \textbf{21.81±4.43}                                  \\
			CTSpine 1K                & 34.96±3.23         & 76.28±9.07         & 38.44±16.18                                        & \textbf{38.59±0.95} & \textbf{86.59±1.78} & \textbf{23.65±2.57}                                  \\
			Mayo2020                  & 32.51±1.91         & 74.95±5.71         & 48.66±12.39                                        & \textbf{33.01±1.50} & \textbf{76.56±4.89} & \textbf{45.36±8.21}                                  \\
			\bottomrule
		\end{tabular}
	}
\end{table}
\subsubsection{Hyperparameters and Computation Time}
\label{Hyperparameters and computation time}
\begin{table}[!htbp]
	\centering
	\caption{Ablation study on the hyper-parameters $n$ and $R_{rand}$.}
	\label{tab:ablation_nR}
	\resizebox{\linewidth}{!}{%
		\begin{tabular}{ccccc}
			\toprule
			\multicolumn{2}{c}{Hyper-parameters} & \multirow{2}{*}{PSNR (dB) $\uparrow$} & \multirow{2}{*}{SSIM (\%) $\uparrow$} & \multirow{2}{*}{RMSE $\downarrow$}\\
			\cmidrule(r){1-2} % 仅在超参数列下画一条短横线
			Fixed & Varying & & \\
			\midrule
			% 第一组：固定 R，变化 n
			\multirow{4}{*}{$R_{rand} = (0.5, 0.6)$} % 假设固定 R 为 (0.5, 0.6)，请替换为您真实的固定值
			& $n=2$ &39.70±1.51  &95.33±1.76  &20.97±3.77\\
			& $n=4$ & \textbf{39.77±1.70} & \textbf{95.35±1.73} & 20.91±3.57\\
			& $n=8$ &39.72±1.40  &95.31±1.64  &20.92±3.48\\
			& $n=10$ &39.74±1.45  &95.31±1.70  &\textbf{20.88±3.59}\\
			\midrule
			% 第二组：固定 n，变化 R
			\multirow{4}{*}{$n=4$} 
			& $R=(0.3, 0.5)$ &39.38±1.08  &\textbf{95.90±0.97}  &21.63±2.66\\
			& $R=(0.4, 0.5)$ &39.71±1.31  &95.64±1.43  &20.93±3.19\\
			& $R=(0.5, 0.6)$ & \textbf{39.77±1.70} & 95.35±1.73 & \textbf{20.91±3.57}\\
			& $R=(0.6, 0.8)$ &39.60±1.56  &94.97±2.07  &21.28±4.00\\
			\bottomrule
		\end{tabular}
	}
\end{table}
\begin{table}[!htbp] 
	\centering
	\caption{Efficiency comparison of different methods.}
	\label{tab:efficiency}
	\footnotesize 
	\setlength{\tabcolsep}{3pt} % 进一步缩小列间距以适应单栏
	\resizebox{\linewidth}{!}{ % 强制缩放至单栏文本宽度
		\begin{tabular}{cccccccccc}
			\toprule
			Method & BM3D & B2U & NBR & Noisier2noise & Noise2Sim & Noise2detail & ZS-N2N & Pixel2Pixel & FrequencyCT \\
			\midrule
			Time (s) $\downarrow$   & 231 & 69600 & 7200 & 29580 & 13500 & 125  & 69            & 52   & \textbf{30} \\
			Params (M) $\downarrow$ & -   & 0.07  & 1.28 & 0.56  & 4.19  & 0.02 & \textbf{0.02} & 0.15 & 0.15 \\
			Mem (GB) $\downarrow$   & -   & 18    & 4    & 11    & 8     & 1.4  & \textbf{0.8}  & 5.5  & 3    \\
			\bottomrule
		\end{tabular}
	}
\end{table}
To demonstrate the rationality of the hyperparameter design, this study was conducted on the Mayo2016 dataset with a 10\% dose, and the specific settings are shown in Table \ref{tab:ablation_nR}. The table shows that different hyperparameters have a relatively small impact on the increase in metrics, indicating the robustness of FrequencyCT. Furthermore, to highlight the processing speed of this method, Table \ref{tab:efficiency} displays the processing time, network parameters, and memory usage. Note that the time is the sum of the total training time and the average testing time.
\section{Discussion}
\label{Discussion and Conclusion}
The core of this work is the innovative concept of generating pseudo-samples directly in the frequency domain. Because the generated samples are derived from the essential frequency behaviors of both noise and true signal, we do not assume noise independence. Instead, we apply an orthogonal perturbation operator to break the inevitable noise correlations. Furthermore, unlike conventional studies that train on the full numerical range of the samples, we accurately identify the changing trend of noise variance in the projection domain and provide the corresponding theoretical proofs. We then apply a truncation process to both the randomly selected inputs and the supervision targets. Although this truncation strategy appears bold, we mathematically prove its feasibility. Extensive experimental results further confirm its effectiveness. In summary, FrequencyCT enriches the methodology of sample generation and improves the interpretability of the model optimization process.

Additionally, readers should note that denoising performance will inevitably degrade under extremely low-dose conditions. This degradation occurs primarily because increasingly complex noise distributions lead to stronger noise correlations. This observation also reveals a limitation of our current method. Under severe noise conditions, the low-frequency regions of the frequency domain will inevitably experience noise interference. Consequently, the reconstructed images will retain some residual noise that cannot be entirely removed. Overcoming this specific challenge will be a primary focus for future research. At the same time, the hyperparameters $n$ and $\left(r_1, r_2\right)$ proposed in this work may not represent the optimal strategy. In future work, we plan to explore simpler pseudo-sample generation strategies and investigate the deeper physical relationships between noise and anatomical structures in each projection data point, with the goal of making the hyperparameters more intelligent and adaptive. Medical imaging must consider not only imaging quality but also timeliness for clinical applications. Although this work falls under the zero-shot category and does not require pre-training, it, like many other methods, processes two-dimensional data. To address this, we propose a concept: samples generated based on the sequence of human body scans can be used to construct pseudo-3D samples, which can significantly improve processing speed—a goal we aim to achieve in the future study. In summary, we believe the frequency domain represents a promising area for research in the field of denoising.
\section{Conclusion}
\label{Conclusion}
In this work, we proposed a novel zero-shot self-supervised method named FrequencyCT. Recognizing the properties of noise in the frequency domain, it innovatively adopted phase-preserving noise and mask perturbation to generate pseudo-samples, which greatly mitigated noise correlation. Meanwhile, through a theoretical analysis of the exponential correlation between noise variance of noisy projections and the underlying true signal, we truncated the generated samples before training to enhance network optimization. Both theoretical and experimental results validated the feasibility of this method for full-range inference.
\appendix
\label{Appendix}
To further illustrate the convergence of truncated training and the positive scaling of the inference to the full numerical range, we propose the following proposition and its proof.

\emph{Proposition 1:} Let $\Omega_{\text{full}} = \{ y \mid y \in [0, V] \}$ be the full numerical range of LDCT projection data, with truncated subset $\Omega_{\text{clamp}} = \{ y \mid y \in [0, T] \}$. For independent noisy observations $y_i$ and corresponding true targets $p_i$, the denoising network $f_\theta(\cdot)$ optimized on $\Omega_{\text{clamp}}$ is guaranteed to converge to the optimal operator $f_{\theta^*}(\cdot)$.

\begin{proof}
According to the Central Limit Theorem, the Poisson distribution asymptotically approximates a Gaussian distribution given a statistically significant photon count \cite{fessler1994penalized}.
	
Given $y_i$ and $f_\theta$, its log-likelihood function $L(\cdot)$ can be described as:
\begin{equation}
		\ln L(\theta)=-\frac{1}{2}\sum_{i}\frac{1}{\sigma_{i}^{2}}(f_{\theta}(y_{i})-p_{i})^{2}+C,
		\label{eq18}
\end{equation}
where $C$ is a constant, maximizing $L$ means minimizing the weighted least squares empirical risk:
\begin{equation}
		\min_{\theta}\mathcal{R}_{WLS}(\theta)=\sum_{i}w_{i}(f_{\theta}(y_{i})-p_{i})^{2},
		\label{eq19}
\end{equation}
comparing Eq. \eqref{eq18} and \eqref{eq19}, we find that the theoretically optimal weights are:
\begin{equation}
		w_i^*=\frac{1}{\sigma_i^2},
		\label{eq20}
\end{equation}
substituting Eq. \eqref{eq9} into the above equation, it yields:
\begin{equation}
		w_i^*=\frac{1}{\sigma_i^2}\approx\frac{I_0}{\exp(p_i)}.
		\label{eq21}
\end{equation}
The above equation reveals that as $\sigma_i^2 \rightarrow \infty$, $w_i^\ast \rightarrow 0$, consistent with Theorem~\ref{theorem} ($\sigma_i^2 \approx \exp(p_i)/I_0$): A larger $p_i$ implies a larger noise variance and thus a smaller optimization weight. This consistency also indicates that the operation in Eq.~\eqref{eq7} conforms to the theory: The truncated subset $\Omega_{\text{clamp}}$ isolates severe noise fluctuations, steering the network toward the optimal convergence operator $f_{\theta^\ast}(\cdot)$.
\end{proof}

\emph{Proposition 2:} During the inference, the subset $\Omega_{\text{clamp}}$ and its associated optimal network $f_{\theta^\ast}(\cdot)$ can be strictly and forwardly scaled to the entire projection space $\Omega_{\text{full}}$.

\begin{proof}
Assume the network parameters are $M$, where the weight matrix of the $m$-th layer is denoted as $W_m$. As described in Section \ref{Training and Inference}, this network architecture is predominantly convolutional and exclusively utilizes non-saturating activation functions $\sigma(\cdot)$, satisfying $\forall\alpha>0$, $\sigma(\alpha x)=\alpha\sigma(x)$. Moreover, the convolution operation is linear: For any input tensor $X$, $W \ast (\alpha X) = \alpha (W \ast X)$. Then the network flow satisfies:
\begin{equation}
		\begin{aligned}
			f_{\theta}(\alpha X) & =W_{M}*\sigma(...\alpha\cdot\sigma(W_{1}*X)) \\
			& =\alpha[W_{M}*\sigma(...\sigma(W_{1}*X))] =\alpha f_{\theta}(X).
		\end{aligned}
		\label{eq22}
\end{equation}
Therefore, for a full-range input $y \in \Omega_{\text{full}}$, let $y_{y>T}$ denote the components exceeding $T$ and $y_{y\le T} \in \Omega_{\text{clamp}}$ those not exceeding $T$. Then there must exists $\alpha > 1$ such that:
\begin{equation}
		y_{y>T}=\alpha y_{y\leq T},
		\label{eq23}
\end{equation}
according to Eq. \eqref{eq22}, there holds:
\begin{equation}
		f_{\theta^*}\left(y_{y>T}\right)=f_{\theta^*}\left(\alpha y_{y\leq T}\right)=\alpha f_{\theta^*}\left(y_{y\leq T}\right).
		\label{eq24}
\end{equation}
As the above analysis shows, under the operation of $f_{\theta^\ast}\left(\cdot\right)$, if $\Omega_{\text{clamp}}$ converges, then $y_{y>T}$ also converges, hence $\Omega_{\text{full}}$ converges. Additionally, setting the threshold to $T=1$ aligns with the underlying physical characteristics of the task. While deep learning vision architectures conventionally employ a symmetric normalized training range of $[-1, 1]$, the physical attenuation mechanisms of CT projection domain inherently preclude negative values. To accommodate this strict non-negativity while providing an optimal numerical scale for neural network optimization, establishing $T=1$ serves as a mathematically and physically consistent boundary for the truncated subset.
\end{proof}

%\bibliographystyle{ieeetr}
%\bibliography{main}

\end{document}